\def\year{2021}\relax
\documentclass[letterpaper]{article} 
\usepackage{aaai21}  
\usepackage{times}  
\usepackage{helvet} 
\usepackage{courier}  
\usepackage[hyphens]{url}  
\usepackage{graphicx} 
\usepackage{natbib}  
\usepackage{caption} 
\usepackage{amsfonts}
\usepackage{amsthm}
\usepackage{amsmath}
\usepackage{amscd}
\usepackage{parskip}
\usepackage{enumitem}
\usepackage{verbatim}
\usepackage{bibentry}
\usepackage{natbib}
\usepackage{booktabs}
\usepackage{thm-restate}
\usepackage{mathtools}
\usepackage{thmtools}
\usepackage{comment}
\usepackage{enumitem}
\usepackage[framemethod=TikZ]{mdframed}

\newcommand{\beq}{\begin{equation}}
\newcommand{\eeq}{\end{equation}}

\newcommand{\beqa}{\begin{eqnarray}}
\newcommand{\eeqa}{\end{eqnarray}}

\newcommand{\beqan}{\begin{eqnarray*}}
\newcommand{\eeqan}{\end{eqnarray*}}

\DeclareMathOperator*{\argmax}{arg\,max}
\DeclareMathOperator*{\expect}{\mathbb{E}}

\newcommand{\statespace}{\mathcal{X}}
\newcommand{\actionspace}{\mathcal{A}}
\newcommand{\valuespace}{\ensuremath{\mathcal{Q}}}
\newcommand{\valuepath}{\ensuremath{\mathcal{Q}}}

\newcommand{\defi}{\doteq}

\usepackage{bbm}

\usepackage{tikz-cd}

\usepackage{mathrsfs}

\usepackage{selectp}

\newtheoremstyle{thmstyle}
  {5pt} 
  {\topsep} 
  {} 
  {} 
  {\bfseries} 
  {.} 
  {.5em} 
  {} 
\theoremstyle{thmstyle}

\newtheorem{definition}{Definition}

\excludecomment{previousVersion}

\title{The Value-Improvement Path: \\Towards Better Representations for Reinforcement Learning}

\author {
        Will Dabney,\textsuperscript{\rm 1}
        Andr\'e Barreto, \textsuperscript{\rm 1}
        Mark Rowland, \textsuperscript{\rm 1} \\
        Robert Dadashi, \textsuperscript{\rm 2}
        John Quan, \textsuperscript{\rm 1}
        Marc G. Bellemare, \textsuperscript{\rm 2}
        David Silver \textsuperscript{\rm 1}\\
}
\affiliations {
    \textsuperscript{\rm 1} DeepMind
    \textsuperscript{\rm 2} Google Research \\
    wdabney@google.com
}

\begin{document}

\maketitle
\begin{abstract}
In value-based reinforcement learning (RL), unlike in supervised learning, the agent faces not a single, stationary, approximation problem, but a sequence of value prediction problems. Each time the policy improves, the nature of the problem changes, shifting both the distribution of states and their values.
In this paper we take a novel perspective, arguing that the value prediction problems faced by an RL agent should not be addressed in isolation, but rather as a single, holistic, prediction problem.
An RL algorithm generates a sequence of policies that, approximately, improve towards the optimal policy. 
We explicitly characterize the associated sequence of value functions and call it the \emph{value-improvement path}. Our main idea is to approximate the value-improvement path holistically, rather than to solely track the value function of the current policy. 
Specifically, we discuss the impact that this holistic view of RL has on representation learning. We demonstrate that a representation that spans the \emph{past} value-improvement path will also provide an accurate value approximation for \emph{future} policy improvements. 
We use this insight to better understand existing approaches to auxiliary tasks and to propose new ones.
To test our hypothesis empirically, we augmented a standard deep RL agent with an auxiliary task of learning the value-improvement path. In a study of Atari 2600 games, the augmented agent achieved approximately double the mean and median performance of the baseline agent.
\end{abstract}

\section{Introduction}\label{sec:intro}

Whether receiving prescriptive feedback (supervised learning), or delayed evaluative feedback (reinforcement learning), machine learning requires generalization from a finite collection of examples to an unseen population. In supervised learning, generalization is sometimes framed as avoiding \textit{overfitting} to a finite data set, and we might use methods such as regularization or cross-validation to accomplish this. In reinforcement learning (RL), generalization is no less important, and we have the same risk of overfitting to limited samples. But, additionally, 
the agent faces a \textit{sequence} of learning problems, as the policy is incrementally improved towards optimality. Ideally the agent should also generalize across these problems. Although other types of generalization play a vital role in all areas of machine learning, we focus on this particular form of generalization due to its unique role in reinforcement learning. We argue that the RL problem should be addressed with this peculiarity in mind.
 
The key problem underlying the ability to approximate any function is \emph{representation learning}. 
Although we are ultimately interested in the optimal value function, it has been shown that a representation specialized to this function may be inadequate for representing the sequence of functions leading to it~\citep{mccallum1996reinforcement,li06towards}. We take this argument one step further and note that we should avoid myopically overfitting the representation to \emph{any} value function in this sequence of functions, as each intermediate value function serves as a mere ``stepping stone'' along the path towards the optimal value function. 

In this paper we explicitly characterize the sequence of value functions produced by RL's policy improvement process, which we call the \textit{value-improvement path}. We prove that the efficacy of representation learning depends upon its ability to represent this path. We use this observation both to construct new algorithms and to understand existing algorithms for representation learning.
 
One common and successful way to approach the representation learning problem is through the use of \textit{auxiliary tasks}: additional prediction problems that shape the representation used by the agent \citep{jaderberg2016reinforcement,bellemare2019geometric}. We suggest that predictions based upon the value-improvement path provide a natural basis for representation learning. Furthermore, we analyze how well existing auxiliary tasks actually span the value-improvement path. We build on this analysis to propose novel auxiliary tasks designed with the value-improvement path in mind.

This paper provides several contributions. First, we characterize the nature of the value-improvement path.
Second, we analyze and discuss both existing and novel auxiliary tasks in relation to their effect on representation generalization. 
Third, we provide theoretical insights that begin to explain the role of the value-improvement path in representation learning for RL. 
Finally, we present results from an extensive experimental study of different auxiliary tasks. This study gives preliminary evidence that representations which accurately approximate the past value-improvement path may better approximate future functions on this path. 

\section{Background}\label{sec:background}

We consider a Markov decision process, or MDP, $M \defi (\statespace, \actionspace, P, r, \gamma)$, with finite state space $\statespace$ and action space $\actionspace$, transition kernel $P : \statespace \times \actionspace \rightarrow \mathscr{P}(\statespace)$, reward function $r: \statespace \times \actionspace \times \statespace \rightarrow \mathbb{R}$, and discount factor $\gamma \in [0, 1)$. 
Given a policy $\pi : \statespace \rightarrow \mathscr{P}(\actionspace)$, the action-value function associated with $\pi$ gives the expected return conditioned on each possible starting state-action pair: $Q^\pi(x, a) \defi \mathbb{E}_\pi\lbrack \sum_{t \geq 0} \gamma^t R_t | X_0 = x, A_0 = a \rbrack$, where $R_t \defi r(X_t, A_t, X_{t+1})$ are rewards. The task of \emph{evaluation} of a policy $\pi$ consists of computing $Q^\pi$. Given a policy $\pi$, we define the associated \emph{Bellman operator} as 
\begin{equation}
    \label{eq:policy_evaluation}
    \mathcal{T}^\pi Q(x, a) \defi \mathbb{E}_{\pi, P}\left[r(x, a, x') + \gamma Q(x', a') \right];
\end{equation}
it is well known that $Q^\pi$ is the fixed point of $\mathcal{T}^\pi$~\citep{puterman94markov}.
The task of \emph{control} consists of finding a policy $\pi^*$ maximizing the associated action-value function $Q^{*} \defi Q^{\pi^*}$. RL algorithms based on dynamic programming approach the control problem by alternating policy evaluation~(\ref{eq:policy_evaluation}) with \emph{policy improvement}, in which $Q^\pi$ is used to compute an improved policy 
\begin{equation}
    \label{eq:policy_improvement}
    \pi'(x) \defi \argmax_a Q^\pi(x,a).
\end{equation}
It can be shown that $Q^{\pi'} \succeq Q^\pi$, that is, $Q^{\pi'}(x,a) \ge Q^\pi(x,a)$ for all $(x,a) \in \statespace \times \actionspace$. The alternation between policy evaluation~(\ref{eq:policy_evaluation}) and policy improvement~(\ref{eq:policy_improvement}) can happen at many levels of granularity. For example, if~(\ref{eq:policy_improvement}) is followed by one application of~(\ref{eq:policy_evaluation}) we have the well known \emph{value iteration} algorithm. If instead we compute $Q^{\pi}$, which corresponds to applying $\mathcal{T}^\pi$ an infinite number of times, we recover \emph{policy iteration}~\citep{howard60dynamic}. Under some mild assumptions the alternation between~(\ref{eq:policy_evaluation}) and~(\ref{eq:policy_improvement}) at any level of granularity converges to $Q^{*}$~\citep{puterman94markov}. We will generically refer to algorithms obtained by alternating one application of~(\ref{eq:policy_improvement}) with $n$ applications of~(\ref{eq:policy_evaluation}), where $n$ is possibly infinite, as \emph{value-based} algorithms.

In RL it is assumed that the agent does not have access to the dynamics of the MDP, and thus the expectation in~(\ref{eq:policy_evaluation}) is replaced by samples from $P(\cdot|x,a)$. Many RL algorithms can thus be understood as stochastic approximations of their dynamic programming counterparts; for example, the stochastic version of value iteration is the well known $Q$-learning algorithm~\citep{watkins1992q}. For the sake of exposition, we will refer to both the state-value function $V^\pi$ and the action-value function $Q^\pi$ simply as \textit{value functions}, using their respective symbols to clarify when needed. 

\begin{figure}
    \centering
    \includegraphics[width=.4\textwidth]{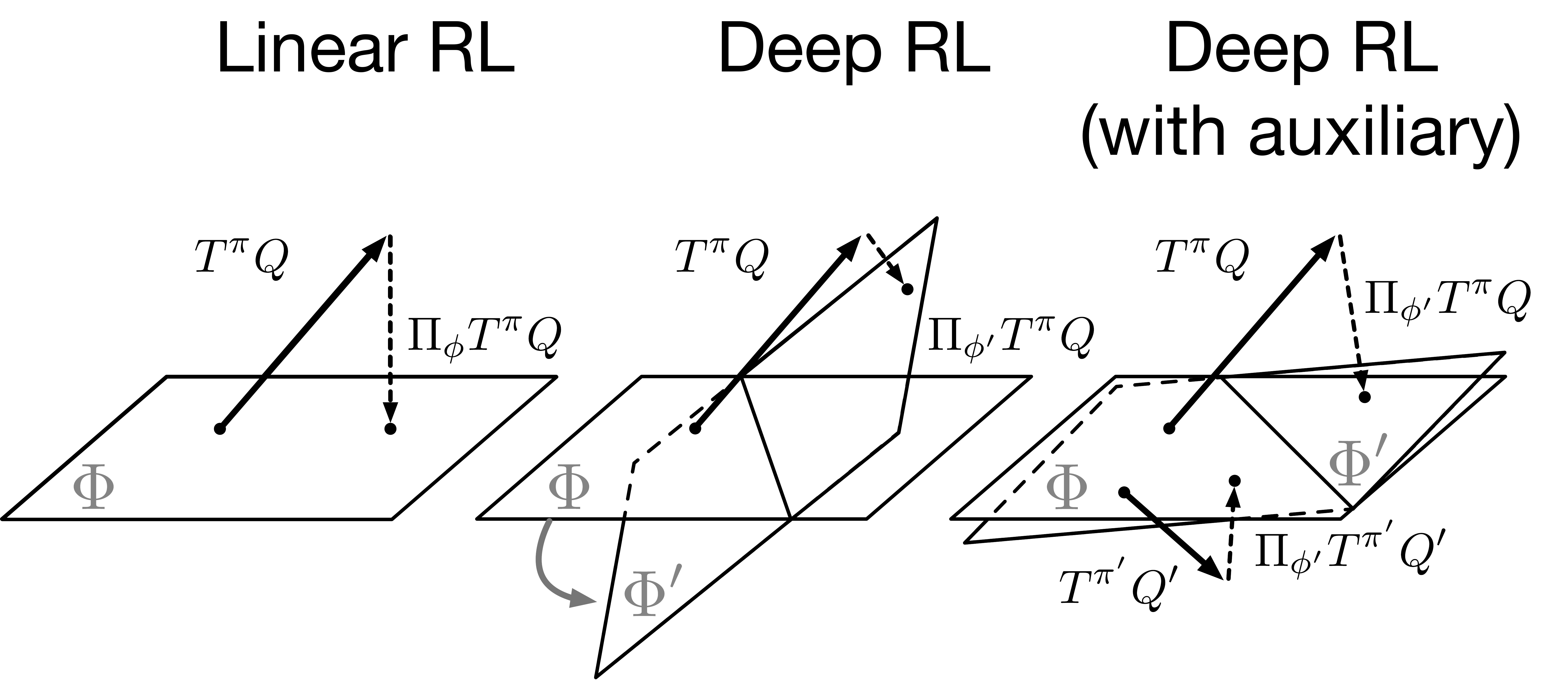}
    \caption{\emph{Linear} approximation in RL projects the Bellman targets onto the space spanned by features. In \emph{Deep RL} the representation itself moves to better fit the targets. However, this can lead to overfitting to the current value function. With \emph{auxiliary tasks} the representation is regularized, reducing representation overfitting. \label{fig:generalization}}
\end{figure}
\textbf{Representation learning.\quad}Usually the state-action space $\statespace \times \actionspace$ is too big to allow for an explicit representation of the functions $Q^\pi$, so one must resort to an approximation $\hat{Q}^\pi$. 
A common approach is to parametrize $\hat{Q}^\pi(x, a)= \phi(x)^\top \theta_a$, where $\phi : \statespace \rightarrow \mathbb{R}^K$ are features, $\theta_a \in \mathbb{R}^K$ are modifiable parameters associated with each $a \in \actionspace$, and $K \in \mathbb{N}$. Given a set of
(possibly non-linear)
features $\phi : \statespace \rightarrow \mathbb{R}^K$, policy evaluation comes down to computing the linear weights $(\theta_a | a \in \actionspace)$ such that $\hat{Q}^\pi \approx Q^\pi$. This involves the projection of the Bellman operator onto the space spanned by $\phi$, $\Pi_{\Phi} \mathcal{T}^\pi \hat Q^\pi$ for $\Phi = \langle \phi \rangle := \text{Span}(\{\phi^\pi_1,\ldots,\phi^\pi_k\})$ (Figure~\ref{fig:generalization}, left).
The problem of constructing the basis functions $\phi$ is known as \emph{representation learning}. Generally $K \ll |\statespace|$, so, for a fixed $\phi$, the space of expressible value functions is much smaller than the space of all possible value functions. Intuitively, then, a well-trained map $\phi$ should be such that $\phi(x)$ captures the salient information of $x \in \statespace$ for return prediction, and ignores irrelevant details.

Representation learning has been extensively studied as a separate step in which one learns features $\phi$ to be later used with linear function approximation. Fixed basis methods such as \textit{tile-coding} \citep{sutton2018reinforcement} and the \textit{Fourier basis} \citep{konidaris2011value} often work well for small input dimensions. However, the field has continually searched for methods that would adapt the representation to the MDP structure. For example, proto-value functions \citep{mahadevan2005proto} attempt to capture the eigenvectors of the matrix $(I - \gamma P^{\pi})^{-1}$, where $P^\pi(x' \mid x) := \mathbb{E}_{\pi} P(x' \mid x, a)$ for policy $\pi$, and are closely related to successor representations \citep{dayan1993improving}, slow-feature analysis \citep{sprekeler2011relation}, and successor features \citep{barreto2017successor}. 

In \emph{deep reinforcement learning} $\phi$ and $\theta_a$ are jointly learned as a (deep) neural network \citep{mnih2015human}. As long as the network has sufficient representational capacity, given enough training experience the learned representation $\phi(x)$ will be able to approximate a policy's value function arbitrarily well. Thus, some may ask if representation learning in RL is still an open problem. We argue for the affirmative, largely due to the need for generalization to future value functions. When training in deep RL the representation itself changes to better fit the value function (Figure~\ref{fig:generalization}, center). This is desirable, but can also lead to highly temporally correlated features \citep{kolter2009regularization}, and a representation that is degenerate, in terms of having limited span. That is, in common machine learning terms, deep RL can, and does, overfit to the current value function.

\textbf{Auxiliary tasks.\quad} It has been argued in the literature that one way to capture the relevant information for a good representation is to learn about many aspects of the world in addition to learning a value function \citep{parr08analysis,zhaoparr16}. One idea in this direction is to define pseudo-rewards, or \emph{cumulants}, $c: \statespace \times \actionspace \times  \statespace \rightarrow \mathbb{R}$, and treat them as actual rewards, either learning the value of a fixed policy or solving the induced control problem in parallel with the solution of the problem of interest \citep{sutton2011horde,jaderberg2016reinforcement}. The intuition here is that these additional tasks, called \emph{auxiliary tasks}, help shape the representation $\phi$ (Figure~\ref{fig:generalization}, right), and thus limit overfitting.

\section{The value-improvement path}\label{sec:value_path}

The fact that in value-based RL policies are computed through~(\ref{eq:policy_improvement}) allows us to think about the problem strictly in terms of value functions. Ultimately, we are interested in the optimal value function $Q^*$, from which an optimal policy can be readily computed. However, unlike in supervised learning, in RL we do not have access to samples of $Q^*$, and to estimate this function the agent must traverse a path across the space of value functions $\valuespace= \{Q^\pi | \pi \in \mathscr{P}(\actionspace)^\statespace\}$.

This special structure of the RL problem creates a number of challenges. Since we do not have direct access to the target function we are trying to approximate, we generally use the approximation itself to build the targets---a strategy sometimes referred to as ``bootstrapping''~\citep{sutton2018reinforcement}. This creates a cyclic dependence of the approximation on itself that can lead to instabilities~\citep{baird1995residual,bertsekas1996neuro,van2018deep}. In fact, many of the techniques currently adopted in deep RL, like target networks and replay buffers, can be interpreted as strategies to ameliorate this instability~\citep{mnih2015human}.

However, in this paper we focus on another challenging aspect of the RL problem that has perhaps been overlooked so far. We argue that, when learning a representation $\phi(x)$, we should keep in mind that we are traversing the space of value functions, and thus over-specializing $\phi(x)$ to a particular value function is analogous to overfitting to a finite dataset in supervised learning. In the same way that we take measures to prevent overfitting in supervised learning, we should adopt strategies to avoid over-specialization of $\phi(x)$. 

But how can we tailor $\phi(x)$ to a set of value functions that is not known in advance? One possible approach is to characterize the \emph{entire} space of value functions \valuespace\ and try to shape $\phi(x)$ in order to represent this space as well as possible. \citet{dadashi2019value} showed that the space of state value functions $\mathcal{V} = \{V^\pi | \pi \in \mathscr{P}(\actionspace)^\statespace\}$ forms a polytope. Based on the theory developed by \citet{dadashi2019value}, it is straightforward to show that, as an affine image of $\mathcal{V}$, \valuespace\ is also a polytope, so we will use ``value polytope'' to generically refer to both $\mathcal{V}$ and \valuespace.

Building on \citeauthor{dadashi2019value}'s insight, \citet{bellemare2019geometric} proposed to shape $\phi(x)$ by learning a set of auxiliary tasks corresponding to value functions that cover the value polytope as well as possible. 
Although \citeauthor{bellemare2019geometric}'s approach is a clear step forward towards recognizing the nature of the approximation problem in RL, we argue that shaping the representation taking the entire value polytope $\valuespace$ (or $\mathcal{V}$) into account may be a stringent requirement in practice.
This is based on the observation that the value functions of interest form a set that is generally much smaller than \valuespace.  

As discussed in Section~\ref{sec:background}, any value-based RL algorithm computes a sequence of functions that, under some assumptions, end in the optimal value function $Q^*$. We will call this trajectory in function space the \emph{value-improvement path}, and formally define it as follows:
\vspace{2mm}
\begin{mdframed}
\vspace{1mm}
\begin{definition} 
 A sequence $\{Q_0, Q_1, ..., Q^*\}$ is called a \emph{value-improvement path} if $Q_{i+1} \succeq Q_{i}$ for $i = 0,1,...$.
\end{definition}
\end{mdframed}

Value-improvement paths are worth investigating because they tend to (approximately) reflect the behavior of algorithms of practical interest. In addition, it might be possible to exploit the structure in this type of sequence to improve the generalization ability of the associated algorithm. In the next section we illustrate these points with a specific example of value-improvement path.

\subsection{A prototypical example of value-improvement path}

In order to provide intuition on the concept of value-improvement path, it might be instructive to consider for a moment the scenario studied by dynamic programming, where it is assumed that the dynamics of the MDP $P(\cdot|x,a)$ are known~\citep{puterman94markov}. This allows for the definition of algorithms whose value-improvement paths can be easily analyzed. Perhaps the dynamic programming algorithm whose value-improvement path is easiest to visualize is policy iteration~\citep{howard60dynamic}. Policy iteration has very simple dynamics: starting from a policy $\pi_0$, compute its value function, $Q^{\pi_0}$, derive an improved policy $\pi_1$ based on~(\ref{eq:policy_improvement}), and so on, until $Q^{\pi_i} = Q^{\pi_{i+1}}$. 

The initial value function $Q^{\pi_0}$ is sufficient to fully define policy iteration's  value path $\{Q^{\pi_0}, Q^{\pi_1}, ..., Q^*\}$; we will thus use $\valuepath^{\pi}$ to refer to policy iteration's value-improvement path starting at $Q^\pi$. The value-improvement path $\valuepath^\pi$ has several interesting properties (see Appendix Figure~\ref{fig:tree}):

  \begin{enumerate}
    \item \label{it:order} {\bf Order}: $\valuepath^{\pi}$ is a totally ordered set, since for any two $Q^{\pi'}, Q^{\pi''} \in \valuepath^{\pi}$ it must be the case that either $Q^{\pi'} \succeq Q^{\pi''}$ or $Q^{\pi'} \preceq Q^{\pi''}$. This is in contrast with \valuespace, which is a partially ordered set.
    \item \label{it:overlap} {\bf Structure}: 
    As long as there is a deterministic way to break ties in~(\ref{eq:policy_improvement}), we can think of the space composed of all policy iteration's value-improvement paths as a tree-like structure in which the optimal value function $Q^*$ is the root, the first level has all the value functions that lead to $\pi^*$ in one application of (\ref{eq:policy_improvement}), and so on. Seen this way, it is clear that two value-improvement paths $\valuepath^{\pi}$ and $\valuepath^{\pi'}$ can intersect at arbitrary levels of the tree, and if they meet in $Q^{\pi_i}$ they overlap from that point up, all the way to the root of the tree. More formally, if $Q^{\pi_i} \in \valuepath^{\pi}$ and $Q^{\pi_i} \in \valuepath^{\pi'}$, given $Q^{\pi_j} \in \valuepath^{\pi}$ such that $Q^{\pi_j} \succeq Q^{\pi_i}$, then $Q^{\pi_j} \in \valuepath^{\pi'}$. 
    \item \label{it:size} {\bf Size}: If $Q^{\pi} \prec Q^{\pi'}$, then $Q^{\pi} \notin \valuepath^{\pi'}$. Also, although $Q^{\pi'} \in \valuepath^{\pi}$ implies that $Q^{\pi'} \succeq Q^{\pi}$, the converse is not necessarily true. Importantly, for a fixed discount factor $\gamma$, the size of any value-improvement path $\valuepath^{\pi}$ is polynomial in $|\statespace|$ and $|\actionspace|$, even though the number of improving policies can be exponentially large in $|\statespace|$ \citep{ye2011simplex}.
\end{enumerate}

The properties above shed some light on the RL representation learning problem. Property~\ref{it:overlap} indicates that the features $\phi(x)$ should always be able to provide a good approximation of $Q^*$---a fact that is not very surprising. Perhaps more insightful is the fact that, although all value-improvement paths end at the same point $Q^*$, the trajectory they define in the value-function space \valuespace\ can be quite distinct. This suggests that the representation learning problem is context-dependent, in the sense that it can change considerably depending on the value function used as a starting point for the policy iteration process. Another interesting fact, implied by Property~\ref{it:size}, is that, once we know $Q^\pi$, we should only care about the value functions $Q^{\pi'} \in \valuepath^{\pi}$. Since this set of value functions is in general much  smaller than the entire polytope $\valuespace$ (see Property~\ref{it:size}), focusing our attention to $\valuepath^{\pi}$ can significantly influence how we approach representation learning.

The structure of the value-improvement path may change depending on how exactly policy evaluation and policy improvement are applied. For example, if policy improvement~(\ref{eq:policy_improvement}) is applied to a subset of the state space $\statespace$ only, there might be many paths from a given function $Q^{\pi_i}$ to the end-point of the path, $Q^*$. This means that the tree structure described in Property~\ref{it:overlap} would be replaced by a directed acyclic graph. Similarly, one should expect the structure of the value-improvement path to change if policy evaluation is not carried out to completion. For example, the value iteration algorithm alternates between a single application of the Bellman operator $\mathcal{T}^\pi$ defined in~(\ref{eq:policy_evaluation}) and one application of the policy improvement operator~(\ref{eq:policy_improvement}); in this case the resulting value-improvement path will also be quite distinct from the one induced by policy iteration---in fact, it has been shown that the intermediate functions obtained by value iteration do not belong to the value polytope $\valuespace$~\citep{dadashi2019value}. Another example of modified policy evaluation is when policy iteration is performed using approximations $\hat{Q}^{\pi} \approx Q^\pi$. In this case the resulting value path $\hat{\valuepath}^\pi \defi \{ \hat{Q}^{\pi_1}, \hat{Q}^{\pi_2}, \dots \}$ may no longer be a value-improvement path~\citep{bertsekas1996neuro}. The value-improvement path also changes when we move from dynamic programming to RL, in which it is assumed that the agent does not have access to the dynamics of the MDP. Since in this case policy evaluation~(\ref{eq:policy_evaluation}) is applied based on samples from $P(\cdot|x,a)$, one has a \emph{distribution} over possible value-improvement paths.

In this paper we will repeatedly refer to policy iteration's value-improvement path $\valuepath^\pi$ as a prototypical example of this type of trajectory in function space. In the same way that knowledge of the structure underlying the true space \valuespace\ helps to shape the representation $\phi(x)$, we argue that the properties of $\valuepath^\pi$ as defined above can help us determine a suitable $\phi(x)$ regardless of the specific way policy evaluation and policy improvement are carried out. We elaborate on this point next. 

\section{Representation learning through the lens of the value-improvement path}
\label{sec:rep_aux_tasks}

\begin{figure*}
    \centering
    \includegraphics[width=\textwidth]{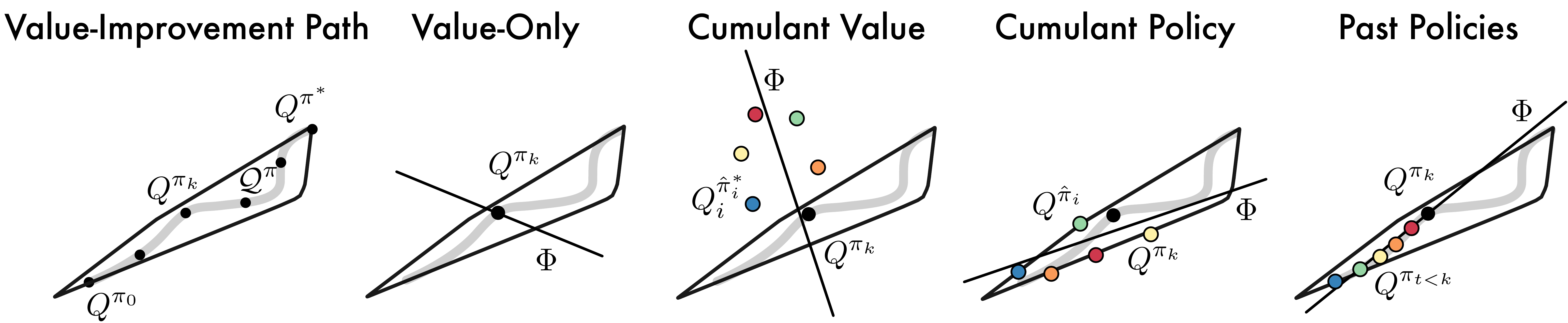}
    \caption{Illustration of the relationship between the value-improvement path, $\mathcal{Q}^\pi$, and auxiliary tasks discussed in this work (shown as colored dots). The outlined space shows the hull of the value polytope, while the gray curve denotes the value-improvement path. Potential representations are shown by a solid line denoted $\Phi$.}
    \label{fig:valuepath}
\end{figure*}

We now revisit the formulation of the representation learning problem in light of the concept of value-improvement path. Currently representation learning in deep RL is tackled in two ways that are, in some sense, the extremes of a spectrum of possibilities. On one extreme of the spectrum we have the common practice of ignoring the special structure of the RL problem and shaping the representation $\phi(x)$ looking only at the current value function. As discussed, this can lead to overfitting. On the other extreme we have the recently-proposed approach of shaping $\phi(x)$ considering the entire value polytope \valuespace, which may not be scalable~\citep{bellemare2019geometric}. Here we propose an intermediate formulation: representation learning in deep RL should be seen as the search for $\phi(x)$ that allows for good approximations of all value functions in an algorithm's value-improvement path. Using again policy iteration as a prototypical reference point, we now formally motivate this objective by restating a result by \citet{munos2003error} in terms of policy iteration's approximate value path $\hat{\valuepath}^\pi$.

Let $\|Q\|^2_{d_\mu} = \sum_{(x, a) \in \statespace\times\actionspace} d_\mu(x, a)Q(x, a)^2$ be the (squared) Euclidean norm on $\mathbb{R}^{\statespace\times\actionspace}$ weighted by the training state-action distribution $d_\mu$, and, for any subspace $U \subseteq \mathbb{R}^{\statespace}$, let $\Pi^{d_{\mu}}_{U} : \mathbb{R}^{\statespace\times\actionspace} \rightarrow \mathbb{R}^{\statespace\times\actionspace}$ denote the orthogonal projection into the subspace $U^\actionspace \subseteq \mathbb{R}^{\statespace\times\actionspace}$ with respect to $\|\cdot\|^2_{d_{\mu}}$ (when it is clear from context, we will drop notational dependence on $d_\mu$).

\begin{restatable}{theorem}{thmLSpolicy iteration}
\label{thm:lspi}
Consider a policy $\pi$, a distribution over state-action pairs, $d_\mu$, and a representation $\phi^\pi: \statespace \to \mathbb{R}^K$. Suppose that policy iteration's approximate value path $\hat \valuepath^\pi$ is well-approximated in the sense that for some $\epsilon \ge 0$ and related distributions $d_{\mu_k}$ (Definition~\ref{def:munos_dist}, Appendix~\ref{sec:proofs}),
\begin{equation}
    \| \Pi_{\langle \phi^\pi \rangle} Q - Q \|_{d_{\mu_{k}}} \le \epsilon,\quad \forall Q \in \hat \valuepath^\pi, k \in [K].
\end{equation}
Then, for $\pi_k$ representing the policy at the $k$\textsuperscript{th} iteration of approximate policy iteration starting with $\pi_0 = \pi$, we have
\begin{equation*}
    \limsup_{k \to \infty} \| Q^* - Q^{\pi_k} \|_{d_\mu} \le \frac{2\gamma\epsilon}{(1 - \gamma)^2}.
\end{equation*}
\end{restatable}

The proofs of our theoretical results can be found in the Appendix. Theorem~\ref{thm:lspi} extends \citeauthor{munos2003error}' result to action-value functions, and indicates how long-term performance of approximate policy iteration is affected by the representation's ability to approximate functions in the value-improvement path. Similar results can be derived for value iteration~\citep{bertsekas1996neuro}.

\subsection{Auxiliary tasks and the value-improvement path}

The representation learning problem can potentially be approached in different ways. However, in this section we focus on a method that is commonly used in practice which can be interpreted as a way of addressing the problem as formulated above: auxiliary tasks. Specifically, we will analyze how well the representation induced by different auxiliary tasks span the value-improvement path. 
We begin with the standard no-auxiliary-task setting and consider methods progressively more aligned with approximating the value-improvement path. As we proceed it will be useful to consider, for each auxiliary task, what subspace the induced representation attempts to capture, and how this compares with that of the value-improvement path. To this end, Figure~\ref{fig:valuepath} gives an illustration of each method.

\textbf{No auxiliary tasks (value-only).\quad}When the training objective is the accurate evaluation of a fixed policy $\pi$, generalization to other policies, improvements or otherwise, can be very poor. Mathematically, under this objective, the aim is to find a representation $\phi^\pi$ and weights $(\theta^\pi_a | a \in \actionspace)$ which obtain a low value for the objective
\begin{align}\label{eq:noAuxTaskObjective}
   \sum_{(x, a) \sim \statespace\times\actionspace} d_\mu(x, a) \left( \phi^\pi(x)^\top \theta^\pi_a \rangle - Q^\pi(x, a) \right)^2 \, ,
\end{align}
where $d_{\mu}$ is the distribution over training state-action pairs. 
If $K \geq |\actionspace|$, it is possible to achieve \emph{zero} error on this objective, by ensuring that the subspace spanned by the coordinates of $\phi$, written $\langle \phi \rangle$, contains each of the action-value functions $Q^\pi(\cdot, a)$ for $a \in \actionspace$. To understand how such a representation generalizes to the approximation of other value functions, let $Q$ be a new value function---corresponding to an improved policy, for example. The optimal approximation to $Q$ using the representation $\phi^\pi$ is given by $\Pi_{\langle \phi^\pi \rangle} Q$. Clearly, any representation $\phi^\pi$ which achieves zero error on the objective \eqref{eq:noAuxTaskObjective} will achieve at least as good an approximation performance of $Q$ as the representation comprising the features $\{Q^\pi(\cdot, a) | a \in \actionspace\}$, but no further guarantees can be given. Thus, the following inequality is tight:
\begin{align*}
    \| \Pi_{\langle \phi^\pi \rangle} Q - Q \| \leq \| \Pi_{\langle Q^{\pi}(\cdot, a) | a \in \actionspace \rangle} Q - Q \| \, .
\end{align*}
\textbf{Cumulant value functions.\quad}Alternatively, we may take on the perspective of  \citet{sutton2011horde} and construct auxiliary tasks by learning optimal value-functions of a diverse collection of cumulants. Let $\{c_1, \ldots, c_n \}$ be a set of $n$ cumulant functions, and for $i = 1, \ldots, n$, let $\hat\pi_i$ be a policy maximizing the value for cumulant $c_i$, and $Q_i^{\hat \pi_i}$ its action-value function under $c_i$. We can think of these auxiliary tasks as capturing a subspace of the space of optimal value functions induced by all possible cumulants: $\langle (I - \gamma P^{\pi^*_c})^{-1} c | \forall c \rangle$. Given infinite representational capacity, adding additional auxiliary tasks would monotonically improve the generalization error. However, when $|\phi|$ is finite, adding auxiliary tasks necessitates a trade-off between approximation errors. In this case, there is no reason to expect the space of all optimal value functions to be well-aligned with the value-improvement path, except for the final point at the optimal policy. For this reason we cannot provide much in the way of generalization guarantees for this auxiliary task. Nonetheless, these can still provide regularization and decorrelation benefits. For example, the UNREAL agent's \textit{pixel control} auxiliary loss is of this type \citep{jaderberg2016reinforcement}.

\textbf{Cumulant policies.\quad}Suppose instead we learn $Q^{\hat\pi_i}_i$ using a separate approximator and use $Q^{\hat \pi_i}$ as our auxiliary task ({\sl i.e.} the evaluation of the auxiliary policy $\hat \pi_i$ on the \emph{true} reward function). This would make sure that the value functions used as auxiliary tasks belong to \valuespace. Clearly, the resulting method will be sensitive to the distribution of policies generated by the cumulant functions.  However, assuming the cumulant functions are sufficiently expressive, the resulting set of auxiliary policies will eventually cover all deterministic policies. Thus, the representation captures the principal components of the value polytope itself~\citep{dadashi2019value}. Because the value-improvement path is a subset of the value polytope, this immediately allows us to apply Theorem~\ref{thm:lspi}. The \textit{adversarial value functions} method provides a principled approach to solving the problem of generating a set of such policies that span the value polytope, but exact solutions can be intractable \citep{bellemare2019geometric}.

The value-improvement path is in general much smaller than the value polytope itself (see Property~\ref{it:size}). Consider for example that for any policy $\pi$ all policies $\hat \pi_i$ with $Q^{\hat\pi_i} \prec Q^{\pi}$ are not in the path $\valuepath^\pi$. The value functions produced by the above two methods may be entirely unrelated to those in the value-improvement path.

\textbf{Past policies in the value path.\quad}Perhaps we can improve generalization error by explicitly restricting the auxiliary value functions to the elements of the value-improvement path. The future policies, and their value functions, are unknown, but we have already passed through some sequence of value-policy pairs during training. Rather than using arbitrary cumulants to generate policies, we can take advantage of the trajectory of improving policies itself to source these auxiliary values. This method involves taking as the auxiliary tasks the value function for the past $k$ policies in our trajectory of policy improvement, $(Q^{\pi_{t-k}}, \ldots, Q^{\pi_{t-1}})$.
\begin{figure*}[t]
    \centering
    \includegraphics[keepaspectratio,width=.9\textwidth]{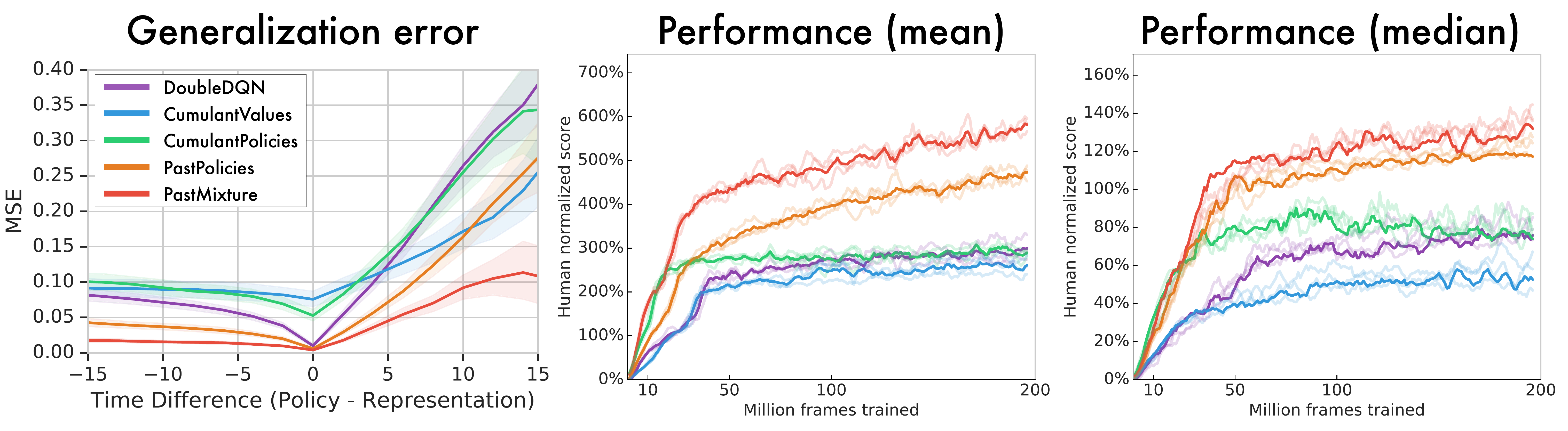}
    \caption{{\bf (Left)} \textit{Generalization:} Mean-squared error between representations $\phi_t$ and value functions $Q_k$ for pairs $t,k$ drawn throughout training. {\bf (Right)} \textit{Performance:} Human-normalized performance on Atari-57. Results averaged over three seeds.}\label{fig:atari_results}
\end{figure*}
We can also consider a softened version of the PastPolicies where each auxiliary task estimates a different mixture over past value-functions. This can be implemented by simply using a different learning rate for each auxiliary task and using their average as the bootstrap target for all tasks. These two auxiliary tasks, PastPolicies and PastMixture, have not been previously proposed, but viewed in the context of the value-improvement path we may expect them to perform well.

The PastPolicies and PastMixtures approaches should be expected to generalize well backwards toward previous policies in the trajectory, as these are what they are being trained on; however, generalization to future values will depend on the exact shape of the value-improvement path.
Interestingly, there is a close connection that can be drawn to optimistic approaches to regret minimization with predictable sequences \citep{rakhlin2013optimization,kalai2005efficient}. At an abstract level, representation learning in the context of value paths can be cast as an online learning problem: at each time step, we must select a collection of features, or equivalently, a subspace of $\mathbb{R}^{\mathcal{X}\times\mathcal{A}}$, and a Q-function is revealed to us, with our loss depending on how well the Q-function is approximable with our chosen subspace. The hypothesis surrounding PastPolicies and PastMixtures is that this problem is not entirely adversarial, and some useful information for future predictions is contained within previous losses; in fact, the PastPolicies approach precisely embodies the \emph{follow-the-leader} approach to this online learning problem.

\subsection{Empirical analysis}\label{sec:empirical}

Our goal in this section is to empirically study the effect of the previously discussed auxiliary tasks on the quality of the learned representation. Our hypothesis is that methods with representations that generalize well along the value-improvement path lead to better long-term control performance. Whereas Theorem~\ref{thm:lspi} proves a more rigorous version of this hypothesis, in this section we test our hypothesis empirically, using a novel evaluation of an agent's representations in terms of the generalization error to past and future value functions in a value path. 


For these experiments, we use the Atari-57 benchmark from the Arcade Learning Environment \citep[ALE]{bellemare13arcade}. We use Double DQN \citep{vanhasselt16deep} as our baseline  non-auxiliary algorithm, and compare with each of the auxiliary tasks: CumulantValues, CumulantPolicies, PastPolicies, and PastMixtures. 
We generated the cumulants for CumulantValues and CumulantPolicies using a random network (details in Appendix~\ref{sec:exp_details}).
Each auxiliary task is trained as a linear function of the last hidden layer of the neural network used by Double DQN, thereby shaping the representation in different ways. To test our hypothesis that tasks aligned with the value-improvement path will lead to improved long-term performance, we first explicitly evaluate this alignment (in hindsight) by measuring approximation error between a representation at one point in time and the value functions along the value path. 

Specifically, we carried out the experiment as follows. While training each agent, for $200$ million environment frames, we saved the current network every $2$ million frames. After training, we evaluated how well, in terms of mean-squared error, each \textit{representation} was able to linearly fit each \textit{value function}. Specifically, we assess how well the representation at time $t$, $\phi_t(s)$, can linearly approximate the value functions $\hat{Q}_k$ for $k = t-15, t-14, ..., t+15$ (Appendix~\ref{sec:exp_details} for details).

Figure~\ref{fig:atari_results} (left) shows a comparison of the generalization errors for each agent on a held-out set of transitions. Each curve can be interpreted as showing the generalization of a representation to other value functions in the value path: negative values correspond to past value functions and positive values to future value functions. Note that in RL we are generally interested in minimizing the latter. Figure~\ref{fig:atari_results} (center, right) shows the human-normalized mean and median scores on Atari-57 (additional results in Appendix~\ref{sec:correlation}).

These results clearly show two trends. First, the methods' ability to generalize to future value functions largely reflects what our analysis based on the value-improvement path would predict. Second, and perhaps more important, the generalization error for future value functions is remarkably, although not perfectly, predictive of long-term performance, corroborating the main argument of this paper that the value-improvement path is the space an RL agent should generalize over. Note that the best performing sets of auxiliary sets, PastPolicies and PastMixtures, are actually novel. It should be straightforward to combine these auxiliary tasks with most value-based algorithms in the literature.

Note that CumulantValue obtains worse long-term performance than other auxiliary tasks. This could (incorrectly) lead us to conclude that auxiliary tasks should only be defined in terms of the MDP's actual reward. Although this is consistent with the main argument of this paper, we believe there are situations in which having other forms of auxiliary tasks may be beneficial. A closer inspection of the results reveals that in some games the CumulantValue task actually performs \emph{best}---notably on those games where exploration is particularly difficult (Appendix Figure~\ref{fig:full}). We speculate that auxiliary tasks based on the value of cumulants may be most useful when the actual reward is sparse. An interesting direction for future work would be to provide a formal justification for cumulant-based auxiliary tasks analogous to the ones provided here for their reward-based counterparts. 

As a final demonstration, we combine PastMixtures with the state-of-the-art Rainbow agent \citep{hessel2018rainbow}. Figure~\ref{fig:rainbow_pm} (Appendix) shows that, despite an existing auxiliary task effect from distributional RL, PastMixtures leads to improved performance of this state-of-the-art agent.

\section{Discussion and related work}

Auxiliary tasks were introduced with \citeauthor{sutton2011horde}'s (\citeyear{sutton2011horde}) Horde architecture, though at the time they were not explicitly aimed at improving an agent's representation. Later, the UNREAL agent introduced an auxiliary task, \textit{pixel control}, to an A3C-like agent, and showed that these significantly improved performance~\citep{jaderberg2016reinforcement}. However, the trade-offs between the auxiliary losses and primary loss required close tuning of hyper-parameters. \citet{fedus2019hyperbolic} proposed learning the distribution of returns for multiple discount factors. Analysis of the relationship between these functions and the value-improvement path is an interesting direction for future work.

\citet{bellemare2019geometric} recently argued for a geometric approach to the representation learning problem based upon the insights surrounding the \textit{value polytope} \citep{dadashi2019value}. Their proposed representational loss takes the form of a minimization of the maximum projection error against a finite set of \textit{adversarial value functions}.

In this paper we focused on how the concept of value-improvement path can be leveraged for representation learning through the use of auxiliary tasks. One can take the ideas presented one step further and treat the entire value-improvement path as a stationary object that can be approximated as a single function~\citep{schaul2015universal,borsa2019universal}. This opens up interesting possibilities in terms of how to adjust such a function: since in general every sample transition can be linked to a specific policy $\pi$ in the value-improvement path, one could think of a training regime that is exclusively ``on-policy''~\citep{sutton2018reinforcement}.

There is also an intriguing connection between the value-improvement path and distributional RL. As discussed, one way to implement the PastMixtures auxiliary tasks is to adopt a different learning rate for each auxiliary task and use their average as the bootstrap target for all tasks. Interestingly, this can be seen as performing an update similar to quantile regression distributional RL, but in which the asymmetric weights are replaced with symmetric weights~\citep{dabney2018implicit,dabney2018distributional}. This connection suggests an explanation as to why distributional RL is so effective in shaping the representation---a question still open in the literature. We discuss this subject in more detail in Appendix~\ref{sec:distributional}.

\section{Conclusions}

In this paper we discussed the problem of representation learning in the context of RL. We argued that one should address this problem keeping in mind that learning a representation for RL is considerably different from the corresponding problem in supervised learning. As a consequence, the common practice of treating each policy evaluation as a conventional supervised learning problem may lead to an over-specialization to intermediate target value functions whose interest for RL is only transient. This is analogous to the problem of overfitting in supervised learning. Under this premise, we presented the following contributions:

\begin{enumerate}
    \item A new formulation of representation learning in RL in terms of an algorithm's value-improvement path.
    \item An interpretation of the commonly-used practice of using auxiliary tasks as a way of addressing the representation learning problem under our new formulation. Specifically, we analyzed how well auxiliary tasks used in the literature, and also new ones, span the value-improvement path. 
    \item Two novel auxiliary tasks inspired by the concept of value-improvement path, PastPolicies and PastMixtures, that showed strong performance and can be readily combined with value-based agents in the literature.
    \item A novel study investigating the effect of auxiliary tasks on the quality of the representation learned. This study is based on an original way of assessing the generalization ability of an approximator that estimates how much it spans a value-improvement path by looking at how well it can represent past and future value functions.  
\end{enumerate}

We believe the insights above shed light on the representation learning problem in the context of RL, allowing a better understanding of practices already used and potentially serving as an inspiration for the design of new  methods.

\clearpage

\subsubsection*{Acknowledgments}
The authors wish to thank colleagues at DeepMind and Google for their encouragement and feedback throughout the process of working on this paper. In particular, to Georg Ostrovski and Doina Precup for discussions and thoughts on a previous draft. As well, we thank the anonymous reviewers for their constructive feedback.


\bibliographystyle{icml2020}
\bibliography{main}

\clearpage

\appendix

\section*{Appendices}

\begin{figure}
    \centering
    \includegraphics[width=.2\textwidth]{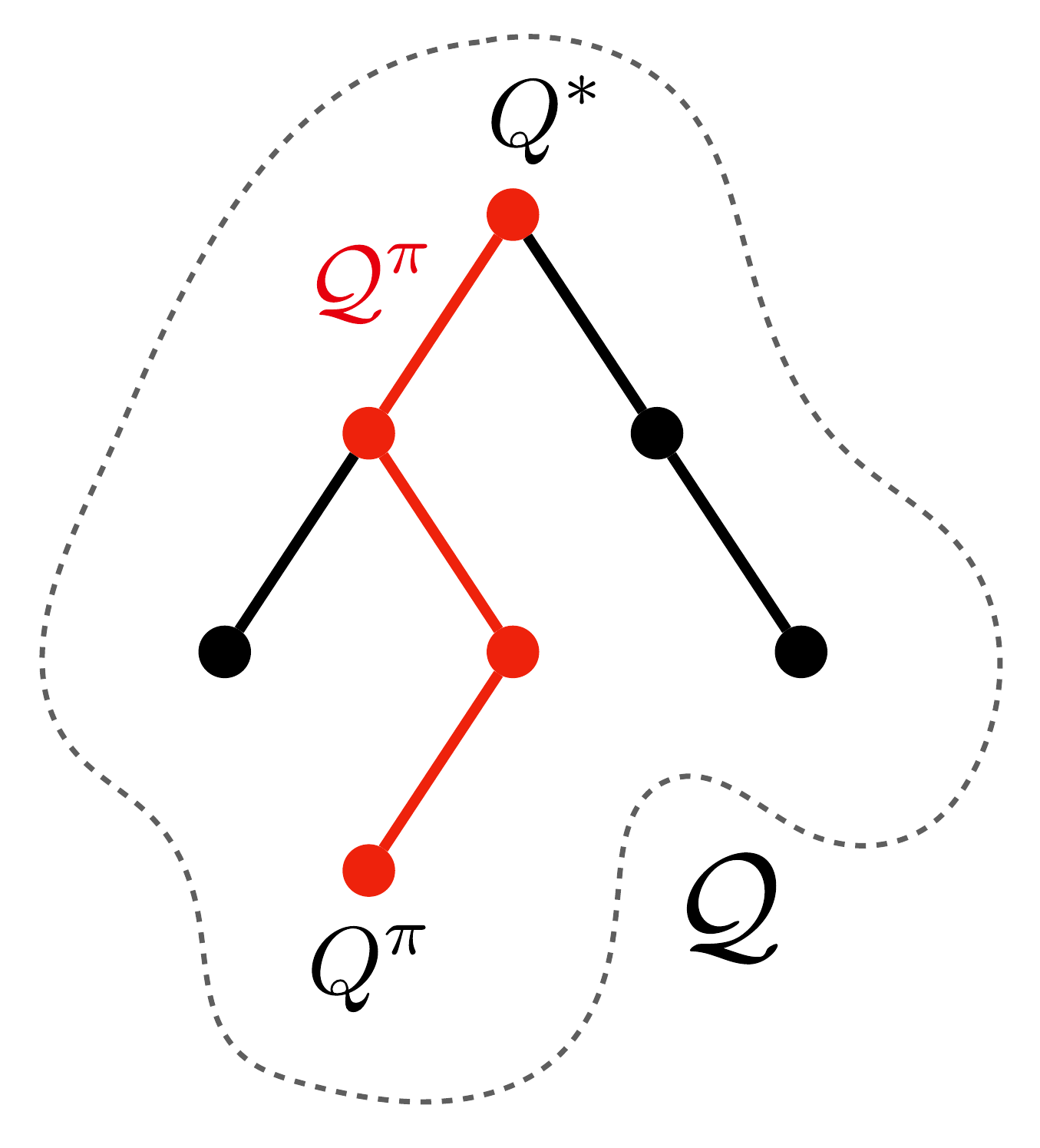}
    \caption{Schematic representation of policy iteration's value-improvement path $\valuepath^\pi$. Each point in the diagram represents a value function $Q^\pi \in \valuespace$. Note any time two value-improvement paths intersect, they merge. \label{fig:tree}}
\end{figure}

\section{Proofs}
\label{sec:proofs}

\begin{definition}\label{def:munos_dist}
    For a distribution $d_\mu$, optimal policy $\pi^*$ and sequence of policies due to policy iteration $\pi_k$, define the stochastic matrices \citep{munos2003error},
    \begin{align*}
        Q_k &= \frac{(1 - \gamma)^2}{2} (I - \gamma P^{\pi^*})^{-1} \times \\
        &\quad \left[ P^{\pi_{k+1}} (I - \gamma P^{\pi_{k+1}})^{-1} + P^{\pi^*} (I - \gamma P^{\pi_k})^{-1}\right],\\
        \tilde Q_k &= \frac{(1 - \gamma)^2}{2} (I - \gamma P^{\pi^*})^{-1} \times \\
        &\quad \left[ P^{\pi_{k+1}} (I - \gamma P^{\pi_{k+1}})^{-1}(I + \gamma P^{\pi_k})  + P^{\pi^*} \right].
    \end{align*}
    Then, the \textit{related distributions} are given by $d_{\mu_k} := d_\mu Q_k$.
\end{definition}

{\bf Proof of Theorem~\ref{thm:lspi}}
\begin{proof}
Our proof extends the results of \citet{munos2003error} (Theorem 1) from state-value functions to action-value functions. First, we bound the norm of the value function approximation error at iteration $k$,
\begin{equation*}
\| V_k - V^{\pi_k} \|^2_{\mu} = \| \mathbb{E}_{\pi_k} Q_k - \mathbb{E}_{\pi_k} Q^{\pi_k} \|^2_{\mu} \le \| Q_k - Q^{\pi_k} \|^2_{\mu \cdot \pi_k}.
\end{equation*}
Thus, bounding the per-iteration approximation error of the action-value function similarly bounds the value function approximation error. 
Second, we lower-bound the sub-optimality of the policy at iteration $k$,
\begin{align*}
\displaystyle \| Q^* - Q^{\pi_k} \|^2_{\mu \cdot \pi_k} &=  \\
\quad \expect_{(x,a) \sim \mu \cdot \pi_k}& (\gamma \expect_{x' \sim P(\cdot \mid x,a)} \left[ V^*(x') - V^{\pi_k}(x') \right])^2,\\
&\le \gamma^2 \| V^* - V^{\pi_k} \|^2_{\mu \cdot \pi_k P}.
\end{align*}
Combining these two with \citeauthor{munos2003error}'s results for value functions yields our result.
\end{proof}

\section{Additional Empirical Results}

\begin{figure*}[t]
    \centering
    \includegraphics[width=\textwidth]{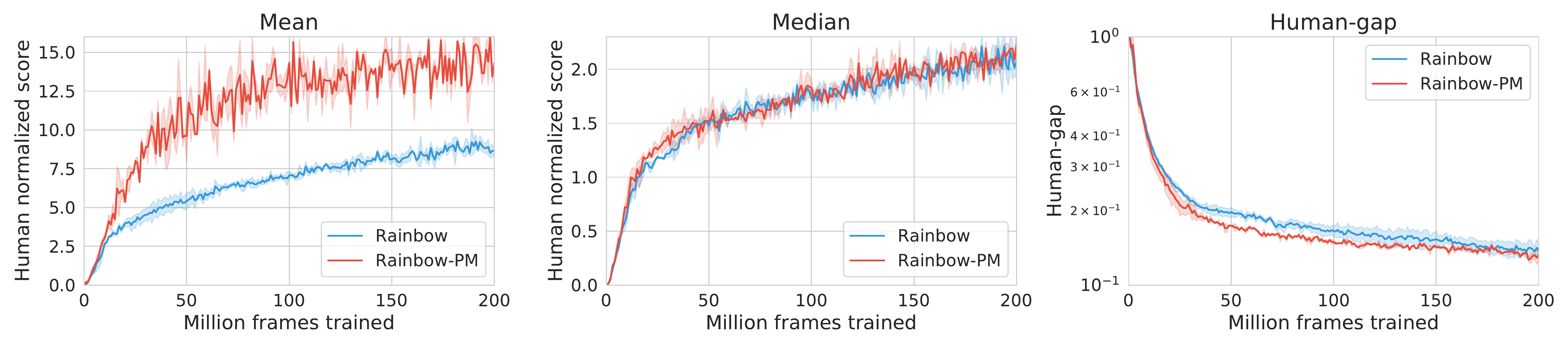}
    \caption{Comparison between Rainbow and Rainbow-PM (with the PastMixtures auxiliary task added). This shows the (\textbf{Left}) mean and (\textbf{Center}) median human normalized scores over Atari-57 benchmark, as well as (\textbf{Right}) the human-gap. }\label{fig:rainbow_pm}
\end{figure*}

In later sections of this appendix we argue that the benefits of distributional RL can be partially understood through their alignment with the value-improvement path. Because of this, we might expect that combining an agent such as Rainbow \cite{hessel2018rainbow}, would not lead to any real improvement in performance. However, this is only a hypothesis, and we might find that the is still some benefit to combining different auxiliary tasks. 

To this end, here we include a final experimental result in which we apply the PastMixtures auxiliary task on top of the Rainbow agent. Figure~\ref{fig:rainbow_pm} shows this comparison between Rainbow and Rainbow-PM, which adds the PastMixtures auxiliary task. We can see that adding this auxiliary task neither helps, nor hurts, the median human-normalized score. However, both the mean human-normalized score and the human-gap are improved. We suggest interpreting these results as showing that distributional RL and PastMixtures capture related information about the state, but due to doing this in different ways can still have some benefit in their combination.

Finally, we would emphasize that experiments such as this on agents with many components (e.g. Rainbow) can make it harder to interpret the empirical results. Specifically, only a full ablation of all components of Rainbow and their combination with PastMixtures would be able to reliably establish where performance benefits are coming from and which components seem to provide overlapping benefits. Thus, we would suggest that the experiments in the main text are the more informative, despite being built upon a non-state-of-the-art agent (DoubleDQN).

\section{Experimental details}
\label{sec:exp_details}

\paragraph{Hyper-parameter sweep}
As a prelude to our primary empirical study we performed a comparable hyper-parameter sweep on all algorithms on a set of six games (Assault, Asterix, Frostbite, Gopher, Seaquest, MsPacman), and used the best performing hyper-parameters for each algorithm when running on the full Atari-57 suite.

For all methods we swept over learning rates $(0.00025, 0.0001, 0.00005, 0.00001)$ and number of auxiliary tasks $(4, 16, 64, 128)$. For PastPolicies, due to memory constraints we instead considered the range $(2, 4, 8, 16)$ for number of tasks. For cumulant-values and cumulant-policies we swept (jointly with other parameters) over the type of cumulant functions generated (i.e. using the output of a random network passed through either a $\mathrm{tanh}$ or $\mathrm{sign}$ function).

\paragraph{Generalization error analysis}
For each agent we performed the generalization analysis on one seed (randomly chosen) and all 57 games. The results in Figure~\ref{fig:atari_results} are averages over all games for each method, with mean-squared errors first normalized to be between $[0, 1]$ using the full range of errors generated by all methods \textit{on that game}. This allows us to show an aggregated plot, despite different games yielding significantly different errors. Note that this means results are averaged over $57$ independent runs, each with up to $100$ saved networks, per algorithm.

The horizontal axis gives the time difference, $k - t$, between the policy whose value function, $Q^{\pi_k}$, provides the prediction target, and the representation, $\phi_t$, upon which we perform linear regression to predict that target. The procedure is as follows: to generate the data we run the $\epsilon$-greedy policy with respect to $Q^{\pi_k}$ for $50000$ frames, with $\epsilon = 0.005$, storing frames and value function outputs of the network. This results in the training examples $\{(s_i, a_i), Q_k(s_i,a_i)\}$ for $k=t-15, t-14, ..., t+15$. We then divide this data set randomly into a training set and testing set with a $90\%/10\%$ split, perform linear regression on the training set, $\min_{w_a} \sum_i \left(\phi_t(s_i)^\top w_a - Q_k(s_i, a_i) \right)^2$ for $k=t-15, t-14, ..., t+15$ and $a \in \actionspace$, and report mean-squared error on the test set. 

\paragraph{Additional details} Other training details match those of \citet{vanhasselt16deep}. In particular, note that we use clipped rewards, end-of-episode life-loss signal, and the standard $\gamma = 0.99$. Experimental results were averages over $3$ seeds.

We use the standard definition for human-normalized score,
$$score = \frac{\text{agent} - \text{random}}{\text{human} - \text{random}},$$
where agent, random, and human each denote the per-game scores for the agent, a random policy, and a human player respectively. The \emph{human-gap} is defined as the average, over games, performance gap from human-level,
$$human\_gap = 1.0 - \mathbb{E}[\min(1.0, score)].$$

\paragraph{Computational resources}
All experiments shown in the main text were implemented using Python and Tensorflow, and were run using a single V100 GPU per 
Rainbow-based agents were implemented in Python using JAX, with each configuration (game, algorithm, hyper-parameter setting) run on a single V100 GPU. Such experiments generally required less than a week of wall-clock time.
\section{Auxiliary tasks}\label{sec:auxtasks}

In the main text we give brief descriptions of each of the auxiliary tasks used in our experiments. Here, we attempt to give a more detailed and reproducible account of how each of these was implemented. For all the auxiliary tasks the additional value functions are parameterized as additional linear heads off of the same shared hidden layer as the primary value function. As these value function heads are part of the main network they use target networks in the same way as is standard with DQN agents. We will use the variable $n$ to denote the number of such additional action-value heads.

\paragraph{Cumulant Value}
The \textit{CumulantValue} task involves learning action-value functions of other reward functions (cumulants). There are two components to this task: (1) the algorithm used to estimate the action-value function, and (2) how the cumulants are produced. For the first, we use the same Double Q-learning update as used for the primary value function. For the second, we use the same neural network as for a standard DQN, but with output dimension of $n$ instead of $|\actionspace|$. For an input state, consisting of stacked frames, $x$ we denote the output of this randomly initialized network as $f_i(x) \in \mathbb{R}$. Given an observed transition from state $x_{t}$ to next state $x_{t+1}$, we define the corresponding cumulant as
\begin{equation*}
    c_i(x_t, x_{t+1}) := \text{tanh}(s \times [f_i(x_{t+1}) - f_i(x_t)]),
\end{equation*}
where $s = 100$ was tuned by hand to produce cumulants across the full range $(-1, 1)$. However, we note that this same end could be achieved by changing the initialization of the network weights. Thus, the auxiliary value functions here are best viewed as estimating the optimal action-value function for their corresponding cumulants, $Q_i^{\hat\pi_i}$. Although these cumulants have the same range of values as the primary (clipped) reward function, they tend to be significantly less sparse. Finally, note that we stop gradients from flowing into the cumulant network, so they are fixed at their randomly initialized values.

Notice that, in general, the values learned by CumulantValue can be significantly different than those of the primary value function. This will force the learned representation to capture information largely unrelated to the main task, which could be viewed as \emph{biasing} the representations learned. On the other hand, as we have seen, in the case of sparse reward tasks, this biased representation can sometimes be beneficial. Additionally, we expect that expanding the network capacity would diminish to negative effects of this misalignment.

\paragraph{Cumulant Policy}
The \textit{CumulantPolicy} task builds upon the CumulantValue task but uses an entirely separate (third) network to estimate the cumulant-specific action-value functions. These action-value functions are used to define a cumulant policy $\hat\pi_i$, which is the greedy policy on the corresponding cumulant value functions. The auxiliary value heads on the main network are trained to \textit{evaluate} these cumulant policies with respect to the primary reward function. That is, the policy $\hat\pi_i$ being targeted is the same as in the CumulantValue task, but the reward function is the true reward instead of the cumulant for which the policy was optimized. We denote these auxiliary action-value functions with $Q^{\hat\pi_i}$.

\paragraph{Past Policies}
In the \textit{PastPolicies} task the auxiliary action-value function heads are trained to evaluate different policies under the true reward function. Here, these auxiliary policies, $\hat\pi_i$, are the greedy policies for the previous target networks of the primary action-value function. Specifically, let $\theta_m$ and $\bar \theta_m = \theta_{m-1}$ denote the network parameters and those of the target network respectively after $m-1$ target network updates. Then, the $n$ auxiliary policies are defined as $\hat\pi_i(x) := \argmax_{a \in \actionspace} Q(x, a; \bar \theta_{m - i})$. Thus, the target policy for auxiliary head $i$ comes from a network that is $i$ target updates old, and thus forms a sliding window over recent past policies of the main action-value function. This is implemented as a sliding window where the $i\text{th}$ action-value head always estimates the value for policy $\hat\pi_i$, corresponding to the $i\text{th}$ most recent target network. This requires maintaining multiple copies of the network parameters, which increases memory usage linearly in the number of past policies considered, and similarly increases the computational cost for each batch update as each past target network must be evaluated.

\paragraph{Past Mixture}
Although similarly motivated, the \textit{PastMixture} task takes quite a different approach. In this case, each auxiliary action-value head tracks the primary reward and policy, but with a different learning rate. That is, each auxiliary head takes as its target the Q-learning target with the bootstrap term composed of the average of all auxiliary head predictions, but the loss for each head is weighted differently ($\alpha_i = (i + 1) / (n + 1)$). This causes each auxiliary head to track changes in value faster or slower than each other, providing different mixtures between the current and past value estimates. All together, PastMixtures can be seen as minimizing the following weighted least-squares problem, with the squared-loss replaced with a Huber-loss in the deep RL setting,

\begin{equation*}
    \sum_i \alpha_i \sum_j (r_t + \gamma \max_a Q_j(x_{t+1}, a, \bar \theta) - \hat Q_i(x_t, a_t, \theta))^2,
\end{equation*}
where $\theta$ and $\bar\theta$ denote the online and target network parameters respectively. For the results in Figure~\ref{fig:rainbow_pm}, the temporal-difference loss above is replaced with the categorical distributional RL loss, but remains a weighted loss over the multiple heads.

Unlike for PastPolicies, PastMixture does not require maintaining additional sets of network parameters. The memory and computational costs for PastMixture very similar to methods such as QR-DQN \citep{dabney2018distributional}.

\section{Correlation analysis}
\label{sec:correlation}

We attempted to further evaluate our claims that generalization error on the value-improvement path is predictive of future performance. To this end, we computed, for each representation and averaged over games, the average future performance of policies between $5$ and $30$ million frames in the future. We also computed the mean-squared error for that representation linearly predicting the future value functions in the same window. In Figure~\ref{fig:correlations} we show a scatter plot of these results and in Table~\ref{fig:correlation_table} we give the Pearson correlation coefficients and resulting $p$-values. 

\begin{figure}[t]
    \centering
    \includegraphics[width=.5\textwidth]{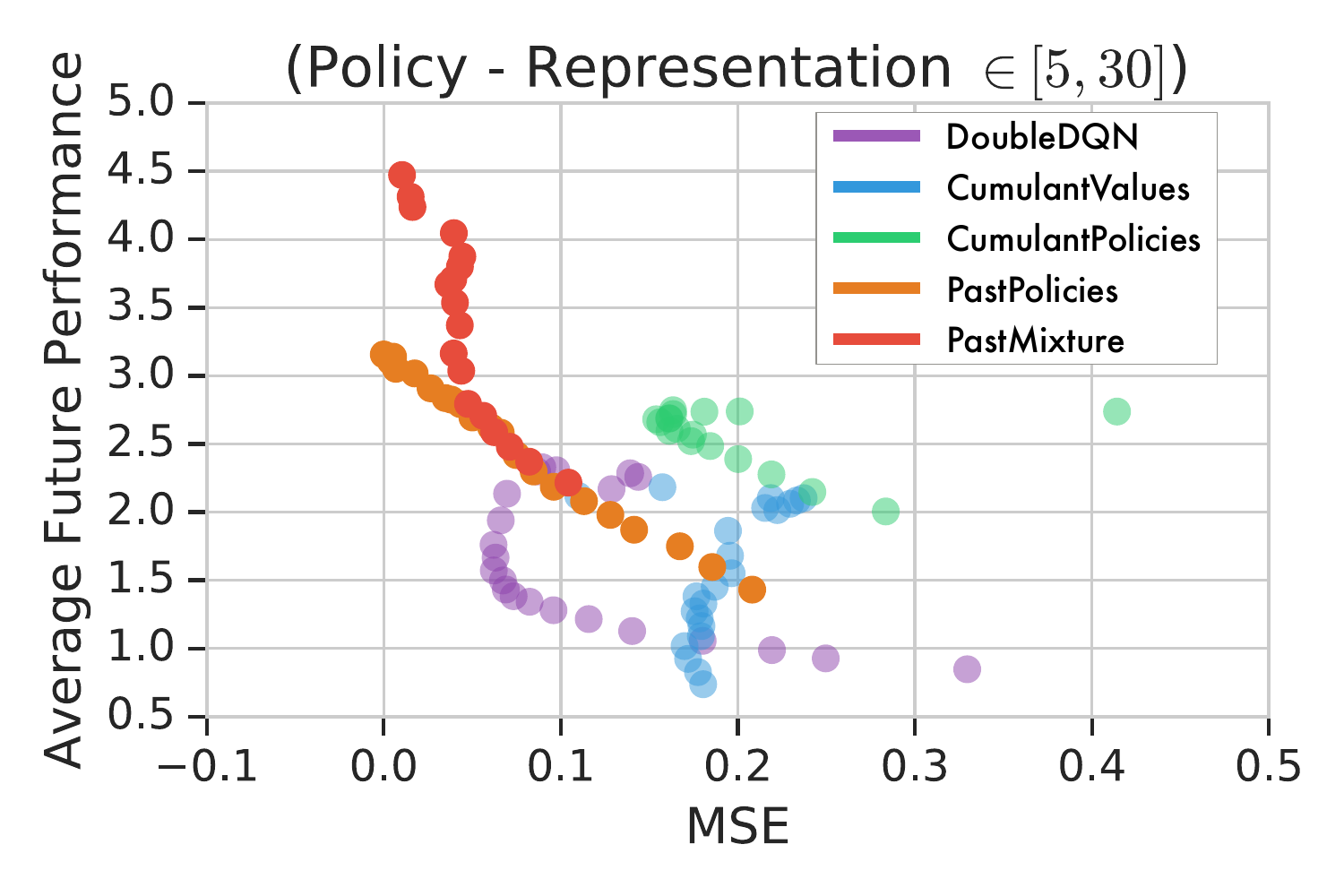}
    \caption{Scatter plot comparing, relative to a feature representation and averaged over all 57 games, an algorithm's future performance (averaged over policies $[5, 30]$ million frames in the future) against the generalization MSE on the same time period. Colors indicate algorithm consistent with the main text. We emphasize, with opacity, the algorithms which are \textit{consistent} with the value-improvement path arguments we introduced.}\label{fig:correlations}
\end{figure}

\begin{table}
    \centering
    \begin{tabular}{lll}
    Algorithm & Pearson coefficient & p-value \\
    \hline
    \hline
    Double DQN          &     -0.518        &    0.014     \\
    Cumulant Values       &     0.370        &    0.091     \\
    Cumulant Policies          &    -0.287         &   0.248 \\
    \hline
    Past Policies         &    -0.995         &   2.38e-21 \\
    Past Mixture          &   -0.895         &   5.50e-7 \\
    \end{tabular}
    \caption{Correlation between generalization errors for future values in the value-improvement path and future performance.}
    \label{fig:correlation_table}
\end{table}

Notice the particularly high degree, and significance, of correlation between generalization error and future performance for the two auxiliary tasks that are conceptually best aligned with the value-improvement path (Past Policies and Past Mixture).

There are limitations of this analysis, primarily it only shows correlation, and so cannot establish a clear causal relationship. That is, both are strongly affected by the current quality of the agent's representation ({\sl i.e.}, better performing agents also have better representations, and stronger future performance). However, we believe it provides additional insight into the nature of the representation learning problem in RL.

\section{Understanding distributional RL through the concept of value-improvement path}
\label{sec:distributional}

In \emph{distributional reinforcement learning} (DRL), the aim is to learn the full distribution of the random returns $Z^\pi(x, a) = \sum_{t \geq 0} \gamma^t R_t | X_0 = x, A_0 = a$ for each $(x, a) \in \statespace \times \actionspace$, rather than just their expected value. Typically, practical approaches to distributional RL involve learning parametrized approximations to these distributions, including categorical distributions with fixed support \citep{bellemare2017distributional} and discrete distributions with varying support \citep{dabney2018distributional}. Recently, it has been shown these approaches can be interpreted as learning a finite collection of auxiliary statistics of the return distribution \citep{rowland2019statistics}, and thus these methods fit naturally in the framework of auxiliary tasks. This interpretation may help to explain their significantly improved empirical performance \citep{hessel2018rainbow,barth-maron2018distributional,dabney2018implicit}.

Distributional RL may confer some of the same benefits as the previous auxiliary task (see Proposition~\ref{prop:mix}). However, viewed in the context of the value-improvement path, distributional RL provides additional avenues for analysis and understanding.

An interesting interpretation of some forms of distributional RL is that they are a natural way of decomposing the representation problem into simpler sub-problems, leading to a sort of \emph{divide-and-conquer} strategy. We will use quantiles to illustrate this point, though the intuition may apply to other forms of distributional RL as well. Recall that a quantile function is defined as $Z^\pi_\tau \defi F^{-1}_\pi(\tau)$, where $\tau \in [0, 1]$ and $F_\pi$ is the cumulative distribution function over possible returns under policy $\pi$. From this definition it should be clear that, given an MDP, all stochastic policies $\pi$ share the same extreme quantile functions $Z^\pi_{0}$ and $Z^\pi_{1}$.\footnote{\textit{Stochastic policies} with non-zero probability on all actions.} This means that, once we have learned $Z^\pi_{0}$ and $Z^\pi_{1}$ for \emph{any} policy, we can ``re-use'' them in the representation of \emph{all} policies $\pi'$ along the value-improvement path. Now, if we increase $\tau$ from $0$ only slightly (or, equivalently, slightly reduce it from $\tau=1$), we would expect the corresponding $Z^\pi_{\tau}$ not to vary much as a function of $\pi$, so these quantile functions should still be quite transferable from one policy to the other. As $\tau \to 0.5$, we would expect $Z^\pi_{\tau}$ to vary more and more as a function of $\pi$, but even when $\tau=0.5$, it should not vary much more than the mean---that is, the value function itself. 

In summary, the ``divide'' step consists in decomposing the problem of approximating the mean $Q^\pi$ into $n$ problems of estimating the quantiles $Z^\pi_\tau$. The ``conquer'' step, in turn, is to average the solutions of the individual problems, that is, $\hat{Q}^\pi = \tfrac{1}{n} \sum_{i=0}^n \hat{Z}^\pi_{\tau_i}$. If each of the sub-problems is indeed simpler than the original, in the sense that it is possible to leverage more information from past estimates of $Z^\pi_\tau$ to represent the current one, then we should expect the representation learned by distributional RL to be useful along the value-improvement path. In what follows we will make these statements more formal and give some examples of scenarios in which the intuition above holds.  

\begin{restatable}{proposition}{propSmooth}\label{prop:smooth}
For a policy $\pi$, let $Z^\pi_\tau: \statespace \times \actionspace \to \mathbb{R}$ be an auxiliary task trained by asymmetric regression (e.g. quantile regression) at threshold $\tau \in [0, 1]$. Assume that $Z^\pi_\tau$ is bounded, monotonically increasing in $\tau$ and Lipschitz continuous such that $\| Z^\pi_\tau - Z^\pi_{\tau'} \| \le \beta | \tau - \tau' |.$
Then,
\begin{equation*}
    \max_{\pi, \pi'} \| Z^\pi_\tau - Z^{\pi'}_\tau \| \le 2\beta \min\{ \tau, 1 - \tau \},\ \forall \tau \in [0, 1].
\end{equation*}
\end{restatable}
\begin{proof}
Begin by observing that, for this class of statistics of the return distribution, for all pair of policies $\pi, \pi'$
\begin{align*}
    & \pi(x, a) > 0, \pi'(x, a) > 0, \forall a \in \actionspace\ \\
    &\implies\ \quad 
    Z^\pi_{\tau = 0} = Z^{\pi'}_{\tau = 0},\ 
    Z^\pi_{\tau = 1} = Z^{\pi'}_{\tau = 1}.
\end{align*}
Let $\tau^\prime \in \{0,1\}$. Then 
\begin{align*}
    \|Z^\pi_\tau - Z^{\pi'}_\tau \| &= \|Z^\pi_\tau - Z^\pi_{\tau^\prime} + Z^{\pi^\prime}_{\tau^\prime} - Z^{\pi'}_\tau \|,\\
    &\le \|Z^\pi_\tau - Z^\pi_{\tau^\prime} \| + \|Z^{\pi'}_\tau - Z^\pi_{\tau^\prime} \|,\\
    &\le 2\beta | \tau - \tau' |.
\end{align*}
Optimizing over $\tau^\prime \in \{0,1\}$ yields the result.
\end{proof}

In Figure~\ref{fig:qpath3d} we show a concrete example of the consequences of Proposition~\ref{prop:smooth}. The MDP is a 3-state chain, with $\gamma = 0.7$ and zero reward except for the terminal state (where $r = 1$). The actions transition left (right) with probability $0.9$ and transition in a random direction with probability $0.1$. The convex hull of the value polytope is shown by the gray surface. We track policies as they interpolate between two fixed policies, varying the mixture $\alpha \in [0.1, 0.3, 0.5, 0.7, 0.9]$. The value functions for these interpolating policies are shown, by color, with the filled points. Finally, we demonstrate the effects of Proposition~\ref{prop:smooth} by plotting, as a color-matched line, the spectrum of quantile functions for each policy. Observe, consistent with our result, that each policy's quantile function converges to the same point as $\tau \to 0$ and $\tau \to 1$, and that they vary maximally as $\tau \approx 0.5$. Furthermore, we observe that they reflect the divide-and-conquer property, where the largest variation is no larger than the variability of the value polytope itself.

\begin{figure}[t]
    \centering
    \includegraphics[width=.45\textwidth]{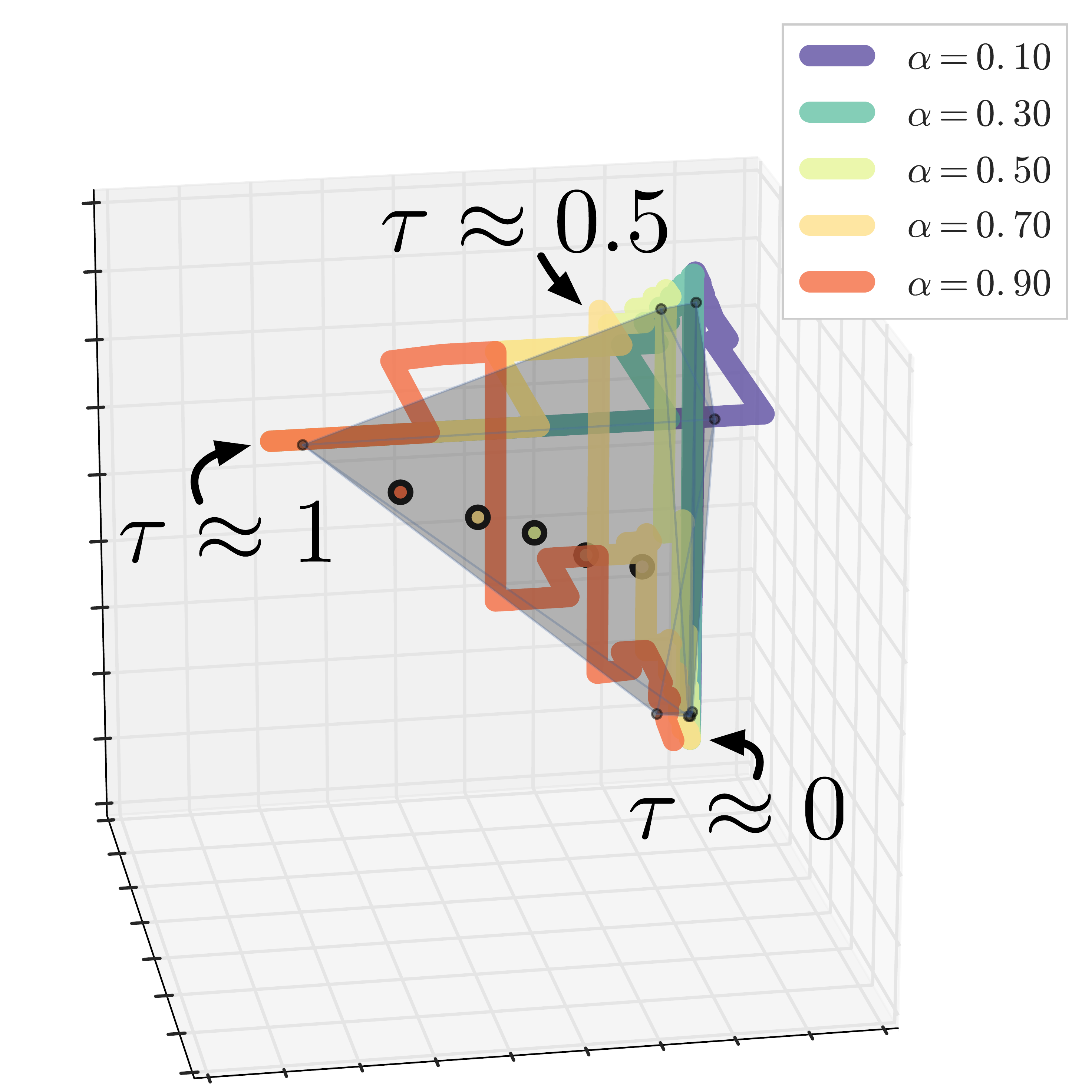}
    \caption{Example of Proposition~\ref{prop:smooth}, demonstrating that the quantile functions for different policies agree at the extremes and vary maximally around the median.}\label{fig:qpath3d}
\end{figure}

\begin{definition}
    A \textbf{deterministic MDP} is an MDP with deterministic transition dynamics and immediate rewards.
\end{definition}

\begin{restatable}{proposition}{propMixture}\label{prop:mix}
Suppose we have the following distributional RL learning update for the action-value distribution $\eta_n$:
$$\eta_{n+1} := (1 - \alpha) \eta_{n} + \alpha \eta^*_{n},$$
where $\eta^*_{n}$ is the state-action target distribution at update $n$ which we assume to be Diracs, and $\alpha \in (0, 1)$. We have that for $\alpha$ small enough, the $N$ quantiles of $\eta_{n}$ match exactly $N$ past values of the target history $\{ \eta^*_{j} \}_{j \leq n}$.
\end{restatable}

\begin{proof}
We can express $\eta_n(x,a)$ as a mixture of past targets:
\begin{align*}
    \eta_n(x,a) &= (1-\alpha)\eta_{n-1}(x,a) + \alpha \eta^*_{n-1}(x,a) \\
    &= \sum^{n-1}_{i=0} \alpha(1- \alpha)^{n-1-i} \eta^*_{i}(x,a) + (1 - \alpha)^n \eta_{0}(x,a)
\end{align*} 

Suppose $\alpha < \frac{1}{N}$, the distribution $\eta_n(x,a)$ is as mixture of $n$ Diracs $\eta^*_{i}(x,a)$ with associated weight $\alpha(1- \alpha)^{n-1-i} < \frac{1}{N}$. Therefore the N quantiles of $\eta_n(x,a)$ match $N$ past target distributions.

\end{proof}

\begin{figure*}
    \centering
    \includegraphics[width=\textwidth]{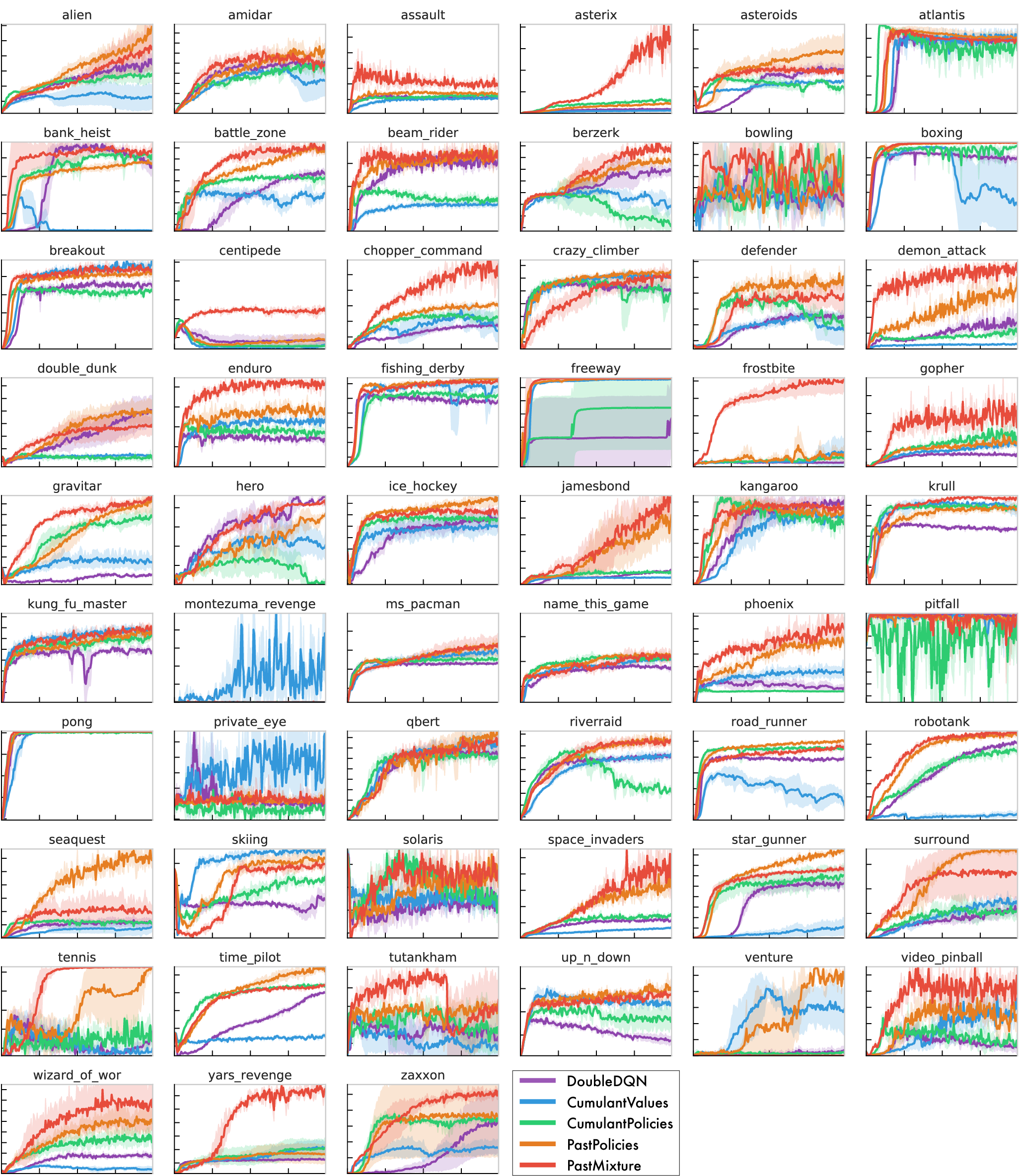}
    \caption{Atari-57 results for all methods, averages over three seeds and error bands show standard deviation.}\label{fig:full}
\end{figure*}

\end{document}